\documentclass{article}

\usepackage{PRIMEarxiv}

\usepackage[utf8]{inputenc} 
\usepackage[T1]{fontenc}    
\usepackage{hyperref}       
\usepackage{url}            
\usepackage{booktabs}       
\usepackage{amsfonts}       
\usepackage{nicefrac}       
\usepackage{microtype}      
\usepackage{lipsum}
\usepackage{fancyhdr}       
\usepackage{graphicx}       
\graphicspath{{media/}}     

\usepackage{algorithmic}
\usepackage{algorithm}
\usepackage{caption}
\usepackage{subcaption}
\usepackage{graphicx}
\usepackage{amsmath,amssymb}
\usepackage{amsthm}
\usepackage{url}
\usepackage{enumerate}
\DeclareMathOperator*{\argmin}{arg\,min}

\newcommand{\R}{\mathbb{R}}

\newcommand{\E}{\mathbb{E}}

\newcommand{\pa}{\text{pa}}
\newcommand{\h}{\textbf{h}}
\newcommand{\x}{\textbf{x}}

\newtheorem{lemma}{Lemma}

\title{Score matching enables causal discovery of nonlinear additive noise models
}

\author{
  Paul Rolland \\
  EPFL \\
  Lausanne, Switzerland \\
  \texttt{paul.rolland@epfl.ch} \\
  \And
  Volkan Cevher \\
  EPFL \\
  Lausanne, Switzerland \\
   \And
  Matth\"aus Kleindessner \\
  Amazon Web Services\\
  T\"ubingen, Germany\\
  \And
  Chris Russel \\
  Amazon Web Services \\ 
  T\"ubingen, Germany\\
  \AND
  Bernhard Sch\"olkopf \\
  Amazon Web Services \\ 
  T\"ubingen, Germany\\
  \And
  Dominik Janzing \\
  Amazon Web Services \\ 
  T\"ubingen, Germany\\
  \And
  Francesco Locatello \\
  Amazon Web Services \\ 
  T\"ubingen, Germany\\
}

\begin{document}
\maketitle

\begin{abstract}
This paper demonstrates how to recover causal graphs from the score of the data distribution in non-linear additive (Gaussian) noise models. Using score matching algorithms as a building block, we show how to design a new generation of scalable causal discovery methods. 
To showcase our approach, we also propose a new efficient method for approximating the score's Jacobian, enabling to recover the causal graph. Empirically, we find that the new algorithm, called SCORE, is competitive with state-of-the-art causal discovery methods while being significantly faster.
\end{abstract}


\section{Introduction}


In this work, we focus on causal discovery from purely observational data, i.e., finding a causal Directed Acyclic Graph (DAG) underlying a given data set. This problem is at the core of causality, since knowledge of the causal graph support the prediction of the effect of interventions ~\cite{peters2017elements,scholkopf2021toward}.


In general, the problem of causal discovery from observational data is ill-posed: there may be several generative models with various causal structures that can produce the same data distribution. Therefore, in order to make the problem well-posed, we need to rely on extra assumptions on the generative process. A popular solution is to assume that the noise injected during the generation of each variable is additive (see equation~\eqref{eq:3.1}). Under additional assumptions on the link functions, it has been shown that such model is identifiable from purely observational data~\cite{peters2014causal}.

\looseness=-1Many causal discovery algorithms maximize a suitable loss function over the set of (DAGs). Unfortunately, solving such problem using classical loss functions is known to be NP-hard~\cite{chickering1996learning}. Therefore, recent methods focused on heuristic approximations, e.g., by using a greedy approach (PC, FCI~\cite{spirtes2000causation, zhang2008completeness}, GES~\cite{chickering2002optimal}, CAM~\cite{buhlmann2014cam} and others~\cite{teyssier2012ordering, larranaga1996learning, singh1993algorithm, cooper1992bayesian, bouckaert1992optimizing}), by expressing the problem as a continuous non-convex optimization problem and applying first-order optimization methods (GraNDAG ~\cite{lachapelle2019gradient}, NOTEARS ~\cite{zheng2018dags}), or by using Reinforcement Learning methods (RL-BIC~\cite{zhu2019causal}, CORL~\cite{wang2021ordering}). 

\looseness=-1There are two distinct aspects that make the search over DAGs difficult: the size of the set of DAGs, which grows super-exponentially with the number of nodes, and the acyclicity constraint. 
In order to reduce the impact of these two difficulties, approaches called order-based methods~\cite{teyssier2012ordering} tackle the problem in two phases. First, they find a certain \textit{topological ordering} of the nodes, such that a node in the ordering can be a parent only of the nodes appearing after it in the same ordering. Second, the graph is constructed respecting the topological ordering and pruning spurious edges, e.g., using sparse regression~\cite{buhlmann2014cam}. While the first step still requires to solve a combinatorial problem, the set of permutations is much smaller than the set of DAGs. Moreover, once a topological order is fixed, the acyclicity constraint is naturally enforced, making the pruning step easier to solve.

The method that we propose is an order-based one, where the topological order is estimated based on an approximation of the \textit{score} of the data distribution. The score of a distribution with a differentiable probability density $p(x)$ is defined as the map $\nabla \log p(x)$.\footnote{The term \textit{score} has been used in the causality literature with a different meaning. Classical works~\cite{chickering2002optimal} use this term referring to the objective of an optimization problem yielding the causal structure as solution. In this paper, the term score means $\nabla \log p(x)$ as in the statistics literature~\cite{wilks1962mathematical}.} 
We show that for a non-linear additive Gaussian noise model, it is possible to identify leaves of the causal graph by analysing its entailed observational score. By sequentially identifying the leaves of the causal graph, and removing the identified leaf variables, one can obtain a complete topological order with a time complexity \textit{linear in the number of nodes}. Classical pruning techniques can then be used in order to obtain the final graph. While the proposed algorithm is designed for additive Gaussian noise models, we show that the main required ingredient for our method to work is the additive structure of the model, rather than the noise type. Hence, we expect similar methods to also be applicable to other types of noise (i.e, non-Gaussian). Closest to our work is LISTEN~\cite{ghoshal2018learning} which can be derived as a special case of our framework in the linear setting as described in the related work section. 

In order to approximate the score of the data distribution from a sample, we exploit and extend recent work on score matching and density gradient estimation~\cite{li2017gradient}. Score approximation methods from observational data have shown success in general machine learning tasks such as generative~\cite{song2019generative} and discriminative models~\cite{zimmermann2021score}, leading to increased interest in developing scalable and efficient solutions. In particular, score-based generative models have shown state-of-the-art performance for image generation~\cite{song2019generative, song2020score, song2020sliced, song2020improved}. As much of the prior work on causal discovery approaches has focused on leveraging machine/deep learning ~\cite{lachapelle2019gradient,zheng2018dags,zhu2019causal,wang2021ordering} to provide a tractable approximation to an NP-hard problem, our work is especially relevant to bridge the gap between provably identifying the causal structure and leveraging advances in deep generative models to scale to large sample sizes and high dimensions.


Hereafter, we summarize our contributions:
\vspace{-2mm}
\begin{itemize}
    \item We start by showing that, in the case of non-linear additive Gaussian noise model, knowing the distribution's score function is sufficient to recover the full causal graph, and we provide a method for doing so. Our approach enjoys a linear complexity in the number of nodes to identify the topological order and introduces a new way of learning causal structure from observational data. To the best of our knowledge, the link between the score function and the causal graph structure established in Lemmata~\ref{lem:1} and~\ref{lem:2} is not only useful, but also novel.
    
    \item We propose a new method for estimating the score's Jacobian over a set of observations, exploiting and extending an existing method based on Stein's identity, which can be of independent interest. This method is then used to design a practical algorithm for estimating the causal topological order.
    
    \item We finally evaluate our proposed algorithm on both synthetic and real world data and show competitive results compared to state-of-the-art methods, while being significantly faster ($10\times$ faster than CAM~\cite{buhlmann2014cam} on 20 nodes graphs and $5\times$ faster than GraN-DAG~\cite{lachapelle2019gradient} on 50 nodes). We also show that our method is robust to noise misspecification and works well when the additive noise is non-Gaussian.
\end{itemize}

\section{Preliminaries}

\subsection{Causal discovery for non-linear additive Gaussian noise models}

Assume that a random variable $X \in \R^d$ is generated using the following model:
\begin{equation} \label{eq:3.1}
    X_i = f_i(\pa_i(X)) + \epsilon_i,
\end{equation}
$i=1, \ldots, d$, where $\pa_i(X)$ selects the coordinates of $X$ which are parents of node $i$ in some DAG. The noise variables $\epsilon_i \sim \mathcal{N}(0, \sigma_i^2)$ 
are jointly independent. The functions $f_j$ are assumed to be twice continuously differentiable and non-linear in every component. That is, if we denote the parents $\pa_j(X)$ of $X_j$ by $X_{k_1}, X_{k_2}, \ldots, X_{k_l}$, then, for all $a=1, \ldots, l$, the function $f_j(x_{k_1}, \ldots, x_{k_{a-1}}, \cdot, x_{k_{a+1}}, \ldots, x_{k_l})$ is assumed to be nonlinear for some $x_{k_1}, \ldots, x_{k_{a-1}}, x_{k_{a+1}}, \ldots, x_{k_l} \in \mathbb{R}^{l-1}$.

This model is known to be identifiable from observational data \cite{peters2014causal}, meaning that it is possible to recover the DAG underlying the generative model~\eqref{eq:3.1} from the knowledge of the joint probability distribution of $X$. In the present work, we aim to identify the causal graph from the score function $\nabla \log p(x)$, which has a one-to-one correspondence with $p(x)$. Hence, any model identifiable from observational data will be identifiable from the knowledge of the data score function.

\subsection{Score matching}

The goal of score matching is to learn the score function $s(x) \equiv
\nabla \log p(x)$ of a distribution with density $p(x)$ given an i.i.d. sample $\{x^k\}_{k=1,\ldots,n}$. In this section, we present a method developed in~\cite{li2017gradient} for estimating the score at the sample points, i.e., approximating $\textbf{G} \equiv (\nabla \log p(x^1), \ldots, \nabla \log p(x^n))^T \in \R^{n\times d}$.

This estimator is based on the well known Stein identity~\cite{stein1972bound}, which states that for any test function $\h:\R^d \rightarrow \R^{d'}$ such that $\lim_{\x \rightarrow \infty} \h(\x) p(\x) = 0$, we have
\begin{equation}\label{eq:3.2}
    \E_p[\h(\x) \nabla \log p(\x)^T + \nabla \h(\x)] = 0,
\end{equation}
where $\nabla \h(\x) \equiv (\nabla h_1(\x), \ldots, \nabla h_{d'}(\x))^T \in \R^{d' \times d}$.

By approximating the expectation in~\eqref{eq:3.2} using Monte Carlo, we obtain
\begin{equation}\label{eq:3.3}
    -\frac{1}{n} \sum_{k=1}^n \h(\x^k) \nabla \log p(\x^k)^T + \text{err} = \frac{1}{n} \sum_{k=1}^n \nabla \h(\x^k),
\end{equation}
where $\text{err}$ is a random error term with mean zero, and which vanishes as $n \rightarrow \infty$ almost surely. By denoting $\textbf{H} = (\h(\x^1), \ldots, \h(\x^n)) \in \R^{d'\times n}$ and $\overline{\nabla \h} = \frac{1}{n} \sum_{k=1}^n \nabla \h(\x^k)$, equation~\eqref{eq:3.3} reads $-\frac{1}{n}\textbf{H} \textbf{G} + \text{err} = \overline{\nabla \h}$. Hence, by using ridge regression, the Stein gradient estimator is defined as:
\begin{align}
    \hat{\textbf{G}}^{\text{Stein}} &\equiv \argmin_{\hat{\textbf{G}}} \|\overline{\nabla \h} + \frac{1}{n}\textbf{H} \hat{\textbf{G}}\|_F^2 + \frac{\eta}{n^2} \|\hat{\textbf{G}}\|_F^2 \\
    &= -(\textbf{K} + \eta \textbf{I})^{-1} \langle \nabla, \textbf{K}\rangle, \label{eq:3.4}
\end{align}
where $\textbf{K} \equiv \textbf{H}^T\textbf{H}$, $\textbf{K}_{ij} = \kappa(\x^i, \x^j) \equiv \h(x^i)^T \h(\x^j)$, $\langle \nabla, \textbf{K}\rangle = n \textbf{H}^T \overline{\nabla \h}$, $\langle \nabla, \textbf{K}\rangle_{ij} = \sum_{k=1}^n \nabla_{x_j^k} \kappa(\x^i, \x^k)$ and $\eta \geq 0$ is a regularisation parameter. One can hence use the kernel trick, and use the estimator~\eqref{eq:3.4} using any kernel $\kappa$ satisfying Stein's identity, such as the RBF kernel as shown in~\cite{liu2016kernelized}.
   
In the present work, we will exploit and extend this approach in order to obtain estimates of the score's Jacobian over the observations.

\section{Causal discovery via score matching}

In this section, we will show how to recover the causal graph from the score function $\nabla \log p(x)$ for a non-linear additive model~\eqref{eq:3.1}. We first design our proposed method in the case where the additive noise is Gaussian, and then discuss extensions to other types of noise.

\subsection{Deduce the causal graph from the score of the data distribution}

Suppose that we have access to enough observational data coming from an additive Gaussian noise model~\eqref{eq:3.1} so that we can accurately approximate the score function of the underlying data distribution. In order to extract information about the graph structure from the score function, let us write it in closed form for a model of the form~\eqref{eq:3.1}. The associated probability distribution is given by
\begin{align*}
    p(x) &= \prod_{i=1}^d p(x_i | \text{pa}_i(x)) \\
    \log p(x) &= \sum_{i=1}^d \log p(x_i | \text{pa}_i(x)) \\
    &= -\frac{1}{2}\sum_{i=1}^d \left(\frac{x_i - f_i(\text{pa}_i(x))}{\sigma_i} \right)^2 - \frac{1}{2}\sum_{i=1}^d \log(2 \pi \sigma_i^2).
\end{align*}

Thus, the score function $s(\x) \equiv \nabla \log p(\x)$ reads
\begin{equation}\label{eq:4.2}
\boxed{
    s_j(x) = -\frac{x_j - f_j(\text{pa}_j(x))}{\sigma_j^2} + \sum_{i \in \text{children}(j)} \frac{\partial f_i}{\partial x_j}(\text{pa}_i(x)) \frac{x_i - f_i(\text{pa}_i(x))}{\sigma_i^2}.
}
\end{equation}

An immediate observation from equation~\eqref{eq:4.2} is that, if $j$ is a leaf, then $s_j(x) = -\frac{x_j - f_j(\text{pa}_j(x))}{\sigma_j^2}$. Since $j \notin \pa_j(x)$, we have that $\frac{\partial s_j(x)}{\partial x_j} = -\frac{1}{\sigma_j^2}$, and hence, it holds that $\text{Var}\left(\frac{\partial s_j(x)}{\partial x_j}\right) = 0$. The following Lemma shows that this condition is also sufficient for $j$ to be a leaf, providing a way to \textit{provably identify a leaf of the graph from the knowledge of the Jacobian of the score function}.

\begin{lemma} \label{lem:1}
Let $p$ be the probability density function of a random variable $X$ defined via a non-linear additive Gaussian noise model~\eqref{eq:3.1}, and let $s(x) = \nabla \log p(x)$ be the associated score function. Then, $\forall j\in\{1,\ldots,d\}$, we have:
\begin{enumerate}[(i)]
    \item $j$ is a leaf $\Leftrightarrow$ $\forall x, \frac{\partial s_j(x)}{\partial x_j} = c$, with $c \in \mathbb{R}$ independent of $x$, i.e., $\text{Var}_X\left[\frac{\partial s_j(X)}{\partial x_j}\right] = 0$.
    \item If $j$ is a leaf, $i$ is a parent of $j$ $\Leftrightarrow$ $s_j(x)$ depends on $x_i$, i.e., $\text{Var}_X\left[\frac{\partial s_j(X)}{\partial x_i}\right] \neq 0$.
\end{enumerate}
\end{lemma}

\begin{proof}
(i) Equation~\eqref{eq:4.2} implies the "$\Rightarrow$" direction as described above.

We prove the other direction by contradiction. Suppose that $j$ is not a leaf and that $\frac{\partial s_j(x)}{\partial x_j} = c$ $\forall x$. We can thus write:
\[
s_j(x) = c x_j + g(x_{-j}),
\]
where $g(x_{-j})$ can depend on any variable but $x_j$. By plugging equation~\eqref{eq:4.2} in $s_j$, we get
\begin{align*}
\frac{f_j(\text{pa}_j(x))}{\sigma_j^2} + \sum_{i \in \text{children}(j)} \frac{\partial f_i}{\partial x_j}(\text{pa}_i(x)) \frac{x_i - f_i(\text{pa}_i(x))}{\sigma_i^2} = \left(c + \frac{1}{\sigma_j^2}\right)x_j + g(x_{-j}).
\end{align*}

Let $i_c$ be a child of node $j$ such that $\forall i \in \text{children}(j)$, $i_c \notin \text{pa}_i$. Such a node always exist since $j$ is not a leaf, and it suffices to pick a child of $j$ appearing at last in some topological order. We then have
\begin{equation} \label{eq:4.3}
\frac{\partial f_{i_c}}{\partial x_j}(\text{pa}_{i_c}(x)) \frac{x_{i_c} - f_{i_c}(\text{pa}_{i_c}(x))}{\sigma_{i_c}^2} - g(x_{-j}) = \left(c + \frac{1}{\sigma_j^2}\right)x_j - \frac{f_j(\text{pa}_j(x))}{\sigma_j^2} - \sum_{i \in \text{children}(j), i\neq i_c} \frac{\partial f_i}{\partial x_j}(\text{pa}_i(x)) \frac{x_i - f_i(\text{pa}_i(x))}{\sigma_i^2}.
\end{equation}

Now, due to the specific choice of $i_c$, we have that the RHS of~\eqref{eq:4.3} does not depend on $x_{i_c}$ (note that we are here speaking about functional dependence on variables, not statistical dependence on a random variable). Hence, we have
\begin{align*}
    &\frac{\partial}{\partial x_{i_c}} \left(\frac{\partial f_{i_c}}{\partial x_j}(\text{pa}_{i_c}(x)) \frac{x_{i_c} - f_{i_c}(\text{pa}_{i_c}(x))}{\sigma_{i_c}^2} - g(x_{-j}) \right) = 0 \\
    &\Rightarrow \frac{\partial f_{i_c}}{\partial x_j} = \sigma_{i_c}^2 \frac{\partial g(x_{-j})}{\partial x_{i_c}}.
\end{align*}

Since $g$ does not depend on $x_j$, this means that $\frac{\partial f_{i_c}}{\partial x_j}$ does not depend on $x_j$ neither, implying that $f_{i_c}$ is linear in $x_j$, contradicting the non-linearity assumption.

(ii) If $j$ is a leaf, then, by equation~\eqref{eq:4.2}, we have:
\[
s_j(x) = -\frac{x_j - f_j(\text{pa}_j(x))}{\sigma_j^2}
\]

If $i$ is not a parent of $j$, then $\frac{\partial s_j}{\partial x_i} \equiv 0$, and hence we have $\text{Var}_X\left[\frac{\partial s_j(x)}{\partial x_i}\right] = 0$. On the other hand, if $i$ is a parent of $j$, then we have $\frac{\partial s_j}{\partial x_i}(x) = \frac{1}{\sigma_j^2} \frac{\partial f_j}{\partial x_i}(\text{pa}_j(x))$. Moreover, since $f_j$ cannot be linear in $x_i$, $\frac{\partial f_j}{\partial x_i}(\text{pa}_j(x))$ cannot be a constant, and hence $\text{Var}_X\left[\frac{\partial s_j(X)}{\partial x_i}\right] \neq 0$.

\end{proof}

\textbf{Discussion:} Lemma~\ref{lem:1} shows that, for non-linear additive Gaussian noise models, leaf nodes (and only leaf nodes) have the property that the associated diagonal element in the score's Jacobian is a constant. This hence provides a way to identify a leaf of the causal graph from the knowledge of the variance of the score's Jacobian diagonal elements. By repeating this method and always removing the identified leaves, we can estimate a full topological order. This procedure is summarized in Algorithm~\ref{alg:1}.  In the following section, we present a new approach, exploiting Stein identities, to compute estimates of the score's Jacobian over a set of samples. 

\begin{algorithm} 
\caption{SCORE-matching causal order search}
\begin{algorithmic} [] \label{alg:1}
\STATE Input: Data matrix $X \in \mathbb{R}^{n \times d}$.
\STATE Initialize $\pi = []$, $\text{nodes} = \{1,\ldots,d\}$
\FOR {$k=1,\ldots,d$}
    \STATE Estimate the score function $s_{nodes} = \nabla \log p_{nodes}$ (for example using Algorithm~\ref{alg:1}).
    \STATE Estimate $V_j = \text{Var}_{X_{nodes}}\left[\frac{\partial s_j(X)}{\partial x_j}\right]$.
    \STATE $l \leftarrow \text{nodes}[\argmin_{j} V_j]$
    \STATE $\pi \leftarrow [l, \pi]$
    \STATE $\text{nodes} \leftarrow \text{nodes} - \{l\}$
    \STATE Remove $l$-th column of $X$
\ENDFOR
\STATE Get the final DAG by pruning the full DAG associated with the topological order $\pi$.
\end{algorithmic}
\end{algorithm}

\subsection{Approximation of the score's Jacobian}

The Stein gradient estimator $\hat{\textbf{G}}^{\text{Stein}}$ enables us to estimate the score function point-wise at each of our sample points. However, according to the previous section, what we need for identifying the graph is an estimate of the Jacobian of the score at all samples, in order to estimate its variance. Since we do not have a functional approximation of the score, we cannot use tricks such as auto-differentiation in order to obtain higher order derivative approximations. In this section, we extend the ideas of Stein based estimator to obtain estimates for the score's Jacobian.

For this purpose, we will use the second-order Stein identity \cite{diaconis2004use, zhu2021hessian}. Assuming that $p$ is twice differentiable, for any $q : \R^d \rightarrow \R$ such that $\lim_{\x \rightarrow \infty} q(\x) p(\x) = 0$ and such that $\E[\nabla^2 q(\x)]$ exists, the second-order Stein identity states that
\begin{equation}
    \E[q(\x) p(\x)^{-1} \nabla^2 p(\x)] = \E[\nabla^2 q(\x)],
\end{equation}
which can be rewritten as
\begin{equation}\label{eq:4.4}
\E[q(\x) \nabla^2 \log p(\x)] = \E[\nabla^2 q(\x) - q(\x) \nabla \log p(\x) \nabla \log p(\x)^T].
\end{equation}

Recall that, in order to identify a leaf of the causal graph, we are only interested in estimating the diagonal elements of the score's Jacobian at the sample points, i.e., $J \equiv (\text{diag}(\nabla^2 \log p(\x^1)), \ldots, \text{diag}(\nabla^2 \log p(\x^n)))^T \in \R^{n\times d}$. Using the diagonal part of the matrix equation~\eqref{eq:4.4} for various test functions gathered in $\h:\R^d \rightarrow \R^{d'}$, we can write

\begin{equation}\label{eq:4.5}
\E[\h(\x) \text{diag}(\nabla^2 \log p(\x))^T] = \E[\nabla^2_{\text{diag}} \h(\x) - \h(\x) \text{diag}((\nabla \log p(\x) \nabla \log p(\x)^T))],
\end{equation}

where $(\nabla^2_{\text{diag}} \h(\x))_{ij} = \frac{\partial^2 h_i(\x)}{\partial x_j^2}$. By approximating the expectations by an empirical average, we obtain, similarly as in~\eqref{eq:3.3},

\begin{equation} \label{eq:4.6}
\frac{1}{n} \sum_{k=1}^n \h(\x^k) \text{diag}(\nabla^2 \log p(\x^k))^T + \text{err} = \frac{1}{n} \sum_{k=1}^n \nabla^2_{\text{diag}} \h(\x^k) - \h(\x^k) \text{diag}(\nabla \log p(\x^k) \nabla \log p(\x^k)^T)).
\end{equation}

By denoting $\textbf{H} = (\h(\x^1), \ldots, \h(\x^n)) \in \R^{d'\times n}$ and $\overline{\nabla^2_{\text{diag}} \h} \equiv \frac{1}{n} \sum_{k=1}^n \nabla^2_{\text{diag}} \h(\x^k)$, equation~\eqref{eq:4.6} reads $\frac{1}{n} \textbf{H} \textbf{J} + \text{err} = \overline{\nabla^2_{\text{diag}} \h} - \frac{1}{n} \textbf{H} \text{diag}(\textbf{G}\textbf{G}^T)$. Hence, by using the Stein gradient estimator for $\textbf{G}$, we define the Stein Hessian estimator as the ridge regression solution of the previous equation, i.e.,

\begin{align}
    &\hat{\textbf{J}}^{\text{Stein}} \equiv \argmin_{\hat{\textbf{J}}} \left \|\frac{1}{n} \textbf{H} \hat{\textbf{J}} + \frac{1}{n} \textbf{H} \text{diag}\left(\hat{\textbf{G}}^{\text{Stein}}\left(\hat{\textbf{G}}^{\text{Stein}}\right)^T\right) - \overline{\nabla^2_{\text{diag}} \h} \right\|_F^2 \nonumber \\
    &\qquad \qquad + \frac{\eta}{n^2} \|\hat{\textbf{J}}\|_F^2 \nonumber \\
    &= -\text{diag}\left(\hat{\textbf{G}}^{\text{Stein}}\left(\hat{\textbf{G}}^{\text{Stein}}\right)^T\right) + (\textbf{K} + \eta \textbf{I})^{-1} \langle \nabla^2_{\text{diag}}, \textbf{K}\rangle, \label{eq:4.7}
\end{align}

where $\textbf{K}_{ij} = \kappa(\x^i, \x^j) \equiv \h(\x^i)^T \h(\x^j)$, $\langle \nabla^2_{\text{diag}}, \textbf{K}\rangle = n \textbf{H}^T \overline{\nabla^2_{\text{diag}} \h}$, $\langle \nabla^2_{\text{diag}}, \textbf{K}\rangle_{ij} = \sum_{i=1}^n \frac{\partial^2 \kappa(\x^i, \x^k)}{\partial (\x^k_j)^2}$ and $\textbf{G}^{\text{Stein}}$ is defined in~\eqref{eq:3.4}. The regularisation parameter $\eta$ lifts the eigenvalues of the same matrix $\textbf{K}$ as in the Stein gradient estimator $\hat{\textbf{G}}^{\text{Stein}}$. We hence decide to use the same parameter for both ridge regression problems.

\paragraph{Choice of kernel} Estimating the score's Jacobian with the method above requires a choice of kernel $\kappa$. A widely used kernel is the RBF kernel $\kappa_s(x,y) = e^{-\frac{\|x-y\|_2^2}{2s^2}}$, which has one parameter $s$ called the bandwidth. This parameter can be estimated from the data to be fitted, using the commonly used median heuristic, i.e., choosing $s$ to be the median of the pairwise distances between vectors in $X$. This estimation procedure even enjoys theoretical convergence properties~\cite{garreau2017large}. Note that, when using Algorithm~\ref{alg:2} for causal discovery in Algorithm~\ref{alg:1}, the kernel lengthscale is re-computed each time a node is removed from the data matrix $X$.

\begin{algorithm} 
\caption{Estimating the Jacobian of the score}
\begin{algorithmic} [Estimating the jacobian of the score] \label{alg:2}
\STATE Input: Data matrix $X \in \mathbb{R}^{n \times d}$, regularisation parameter $\eta > 0$.
\STATE $s \leftarrow \text{median}(\{\|x_i - x_j\|_2: i, j=1,\ldots,n, x_k = X[k,:]\})$.
\STATE Compute $\hat{\textbf{J}}^{\text{Stein}}$ using RBF kernel $\kappa_s$, regularisation parameter $\eta$ and data matrix $X$ based on~\eqref{eq:4.7}.
\end{algorithmic}
\end{algorithm}

\paragraph{Algorithm complexity} Estimating the topological order requires inverting $d$ times an $n\times n$ kernel matrix, hence the complexity is $\mathcal{O}(dn^3)$ (and could be improved using, e.g., Strassen's algorithm~\cite{strassen1969gaussian}). Including the pruning step, the final complexity is hence $\mathcal{O}(dn^3 + d r(n,d))$ where $r(n,d)$ is the complexity of fitting a generalized additive model using $n$ data points in $d$ dimensions. In comparison, the complexity of CAM is $\mathcal{O}(d^2 r(n,d))$. The total computational complexity of GraNDAG is not discussed in~\cite{lachapelle2019gradient}; it is difficult to specify it since it depends on the number of iterations used in the Augmented Lagrangian method, which may depend on the dimension and number of samples. However, GraNDAG is particularly slow due the computation of the acyclicity constraint at each iteration, which requires computing the exponential of a $d\times d$ matrix, taking $\mathcal{O}(d^3)$ operations.

In practice, in our method,  the time for estimating the topological order is much smaller than the time for pruning it ($30\%$ of the total time for $(d,n) = (20, 1000)$ and $5\%$ of the total time for $(d,n) = (50, 1000)$). In comparison, CAM spends most of the time estimating the topological order (more than $95\%$ of the total time in all tested scenarios). Hence, we expect the dominant term in our method's time complexity to be $d r(n,d)$, thus improving upon CAM's complexity. Moreover, in the case where $n$ becomes very large, it is possible to use kernel approximation methods to reduce the time complexity of our method~\cite{si2014memory}.

\subsection{Extension to non-Gaussian additive noise models}

In the previous section, we exploited the structure of the additive Gaussian noise model to deduce the causal graph from the score function ~\eqref{eq:3.1}. Actually, the main ingredient required in our analysis is the additive structure. Indeed, for any additive noise model (including non-Gaussian noise), the score function has a similar structure as in~\eqref{eq:4.2}. 

\begin{lemma}\label{lem:2}
Suppose that the random variable $X$ is generated from~\eqref{eq:3.1} where the noise variables $\epsilon_i$ are i.i.d. with smooth probability distribution function $p^\epsilon$. Then, the score function of $X$ can be written as follows:
\begin{equation}
    s_j(\x) = \frac{d \log p^\epsilon}{d x} (x_j - f_j(\text{pa}_j(\x)))
    - \sum_{i \in \text{children}(j)} \frac{\partial f_i}{\partial x_j}(\text{pa}_i(\x)) \frac{d \log p^\epsilon}{d x} (x_i - f_i(\text{pa}_i(\x))).
\end{equation}
\end{lemma}
The decomposition of the score's components $j$ into a common term $\frac{d \log p^\epsilon}{d x} (x_j - f_j(\text{pa}_j(\x)))$ and a term involving only the parents of the node $j$ is hence characteristic of general additive noise models. Recall that our method identifies leaves by identifying non-linearity in the components of the score. When the common term is linear in $x_j$, as it is the case with Gaussian noise, the second term is the only one carrying non-linearities, and the leaves can hence be perfectly identified with this method (see Lemma~\ref{lem:1}). However, intuitively speaking, even when the noise is non-Gaussian, i.e., when the common term carries non-linearities, the second term still carries non-linearities proportionally to the number of parents of node $j$. Hence, we may expect that the proposed algorithm can work in the case of general additive models, even when the noise is non-Gaussian. While this does not provide a formal identifiability statement, we will show in the experimental section that our proposed method outperforms other state-of-the-art algorithms on non-Gaussian additive models.



\section{Numerical experiments}

We now apply Algorithm~\ref{alg:1} with Algorithm~\ref{alg:2} as score estimator to synthetic and real-world datasets and compare its performance to state-of-the-art methods, such as CAM~\cite{buhlmann2014cam} and GraNDAG~\cite{lachapelle2019gradient}. The other methods (NOTEARS, PC, FCI, GES,...) are omitted since they perform much worse ~\cite{buhlmann2014cam, lachapelle2019gradient}. Recent work \cite{reisach2021beware} warned about the fact that simulated data sometimes lead to scenari where a topological order can simply be estimated by sorting the nodes variances. In order to defend ourselves against this, we randomly generate the noise variances in the generative model, and show that the estimated order when sorting the variance is much worse than the one estimated by Algorithm~\ref{alg:1}.

\subsection{Synthetic data}

We test our algorithm on synthetic data generated from a non-linear additive Gaussian noise model~\eqref{eq:3.1}. Mimicking~\cite{lachapelle2019gradient, zhu2019causal}, we generate the link functions $f_i$ by sampling Gaussian processes with a unit bandwidth RBF kernel. The noise variances $\sigma_i^2$ are independently sampled uniformly in $[0.4, 0.8]$. The causal graph is generated using the Erd\"os-R\'enyi model~\cite{ErdosRenyi}. For a fixed number of nodes $d$, we vary the sparsity of the sampled graph by setting the average number of edges to be either $d$ (ER1) or $4d$ (ER4). Moreover, to test the robustness of the algorithm against noise type misspecification, we also generate data with Laplace noise instead of Gaussian noise. Additional experiments, using Gumbel noise and scale free graphs~\cite{barabasi1999emergence} can be found in Appendix~\ref{app:1}.

For each method, we compute the structural Hamming distance (SHD) between the output and the true causal graph, which counts the number of missing, falsely detected or reversed edges, as well as the structural intervention distance (SID) \cite{peters2015structural} which counts the number of interventional distribution which would be miscalculated using the chosen causal graph. 

For all order-based causal discovery methods, we always apply the same pruning procedure, i.e., CAM with the same cutoff parameter of $0.001$. Moreover, we compute a quantity measuring how well the topological order is estimated. For an ordering $\pi$, and a target adjacency matrix $A$, we define the topological order divergence $D_{top}(\pi, A)$ as
\[
D_{top}(\pi, A) = \sum_{i=1}^d \sum_{j:\pi_i > \pi_j} A_{ij}.
\]
If $\pi$ is a correct topological order for $A$, then $D_{top}(\pi, A) = 0$. Otherwise, $D_{top}(\pi, A)$ counts the number of edges that cannot be recovered due to the choice of topological order. It hence provides a lower bound on the SHD of the final algorithm (irrespective of the pruning method).

The results of the synthetic experiments are shown in Tables~\ref{tab:1} to~\ref{tab:6}. The computed quantities are averages over $10$ independent runs. We can see that, for sparser graphs (ER1), our method performs similarly as the best method CAM. However, for denser graphs (ER4), our method performs better, and in particular seems to estimate a better topological order, since the $D_{top}$ value is smaller. For $50$ nodes graphs, the two best methods are CAM and ours, which both perform similarly. Note that, in order to run it within a reasonable time frame, we had to restrict the maximum number of neighbours, hence providing a sparsity prior to the algorithm, which fits the correct graph in this situation, since sparse Erd\"os-Renyi graphs usually do not contain high degree nodes. Since we restricted the number of neighbours in the graph, the order finding part of CAM does not yield a single topological order, hence we could not compute $D_{top}$ in this setting. We also observe that the topological ordering resulting from sorting the variances (VarSort) is much worse in general than with all other methods, showing that finding a topological order for the generated datasets is not a trivial task. Finally, we observe that our method is quite robust to noise misspecification, since the accuracy remains very similar for Laplace noise.

In terms of running time (Table~\ref{tab:time}), we see that our method is significantly faster. Actually, most of the time ($95\%$ for $d=50$) is spent on pruning the final DAG using CAM.

\begin{table*}[h!]
\centering
\caption{Synthetic experiment for $d=10$ with Gaussian noise}
    \begin{tabular}{c|cccccc} 
    \hline
    & ER1 & & & ER4 & & \\
    \hline
    & SHD & SID & $D_{top}(\pi, A)$ & SHD & SID & $D_{top}(\pi, A)$ \\
    \hline
        SCORE (ours) & $\bf{1.1 \pm 0.9}$ & $\bf{4.5 \pm 5.3}$ & $\bf{0.4 \pm 0.6}$ & $\bf{19.5 \pm 2.9}$ & $\bf{35.0 \pm 9.1}$ & $\bf{0.3 \pm 0.3}$   \\
        CAM & $1.7 \pm 1.0$ & $6.4 \pm 4.2$ & $\bf{0.4 \pm 0.5}$ & $24.4 \pm 3.1$ & $45.2 \pm 10.2$ & $4.4 \pm 3.2$ \\
        GraN-DAG & $1.5 \pm 1.4$ & $6.5 \pm 7.2$ & $-$ & $22.2 \pm 2.6$ & $42.0 \pm 6.2$ & $-$ \\
        VarSort & $-$ & $-$ & $1.9 \pm 1.1$ & $-$ & $-$ & $9.7 \pm 3.1$ \\
    \hline
    \end{tabular}
    \label{tab:1}
\end{table*}

\begin{table*}[h!]
    \centering
    \caption{Synthetic experiment for $d=20$ with Gaussian noise}
    \begin{tabular}{c|cccccc} 
    \hline
    & ER1 & & & ER4 & & \\
    \hline
    & SHD & SID & $D_{top}(\pi, A)$ & SHD & SID & $D_{top}(\pi, A)$ \\
    \hline
        SCORE (ours) & $\bf{2.6 \pm 1.9}$ & $\bf{9.9 \pm 8.5}$ & $1.2 \pm 1.7$ & $\bf{47.5 \pm 4.5}$ & $\bf{177.5 \pm 11.6}$ & $\bf{3.1 \pm 1.5}$   \\
        CAM & $3.5 \pm 1.6$ & $14.3 \pm 9.8$ & $\bf{0.8 \pm 1.0}$ & $54.2 \pm 5.4$ & $201.9 \pm 29.0$ & $13.6 \pm 6.9$ \\
        GraN-DAG & $7.6 \pm 4.2$ & $31.6 \pm 22.7$ & $-$ & $49.3 \pm 4.5$ & $211.4 \pm 36.6$ & $-$ \\
        VarSort & $-$ & $-$ & $3.7 \pm 1.6$ & $-$ & $-$ & $18.3 \pm 6.7$ \\
    \hline
    \end{tabular}
    
    \label{tab:2}
\end{table*}

\begin{table*}[h!]
    \centering
    \caption{Synthetic experiment for $d=50$ with Gaussian noise}
    \begin{tabular}{c|cccccc} 
    \hline
    & ER1 & & & ER4 & & \\
    \hline
    & SHD & SID & $D_{top}(\pi, A)$ & SHD & SID & $D_{top}(\pi, A)$ \\
    \hline
        SCORE (ours) & $10.4 \pm 3.9$ & $\bf{50.9 \pm 32.9}$ & $3.9 \pm 2.4$ & $\bf{131.5 \pm 7.5}$ & $\bf{1262 \pm 110}$ & $16.3 \pm 6.1$   \\
        CAM & $\bf{8.3 \pm 2.9}$ & $53.7 \pm 31.9$ & $-$ & $140.8 \pm 5.5$ & $1337 \pm 94$ & $-$  \\
        GraN-DAG & $20.2 \pm 6.1$ & $135.3 \pm 45.9$ & $-$ & $140.8 \pm 9.5$ & $1432 \pm 110$ & $-$ \\
        VarSort & $-$ & $-$ & $8.8 \pm 3.0$ & $-$ & $-$ & $43.3 \pm 9.7$ \\
    \hline
    \end{tabular}
    \label{tab:3}
\end{table*}

\begin{table*}[h!]
    \centering
    \caption{Synthetic experiment for $d=10$ with Laplace noise}
    \begin{tabular}{c|cccccc} 
    \hline
    & ER1 & & & ER4 & & \\
    \hline
    & SHD & SID & $D_{top}(\pi, A)$ & SHD & SID & $D_{top}(\pi, A)$ \\
    \hline
        SCORE (ours) & $1.4 \pm 0.8$ & $4.5 \pm 4.7$ & $0.8 \pm 0.7$ & $\bf{19.6 \pm 2.5}$ & $\bf{31.9 \pm 7.9}$ & $\bf{0.2 \pm 0.4}$   \\
        CAM & $1.5 \pm 1.3$ & $6.1 \pm 6.5$ & $\bf{0.5 \pm 0.5}$ & $24.4 \pm 1.5$ & $44.4 \pm 8.1$ & $1.5 \pm 1.6$ \\
        GraN-DAG & $\bf{1.3 \pm 1.4}$ & $\bf{4.4 \pm 4.9}$ & $-$ & $20.3 \pm 2.7$ & $39.3 \pm 13.0$ & $-$ \\
        VarSort & $-$ & $-$ & $1.6 \pm 1.3$ & $-$ & $-$ & $7.2 \pm 2.3$ \\
    \hline
    \end{tabular}
    \label{tab:4}
\end{table*}

\begin{table*}[h!]
    \centering
    \caption{Synthetic experiment for $d=20$ with Laplace noise}
    \begin{tabular}{c|cccccc} 
    \hline
    & ER1 & & & ER4 & & \\
    \hline
    & SHD & SID & $D_{top}(\pi, A)$ & SHD & SID & $D_{top}(\pi, A)$ \\
    \hline
        SCORE (ours) & $\bf{1.6 \pm 1.2}$ & $\bf{6.8 \pm 11.4}$ & $0.5 \pm 0.9$ & $\bf{48.0 \pm 4.0}$ & $\bf{199.8 \pm 21.4}$ & $\bf{4.9 \pm 1.8}$   \\
        CAM & $2.3 \pm 1.4$ & $10.0 \pm 7.0$ & $\bf{0.3 \pm 0.5}$ & $52.4 \pm 3.9$ & $208.7 \pm 17.5$ & $11.6 \pm 7.9$ \\
        GraN-DAG & $4.9 \pm 2.1$ & $27.5 \pm 13.2$ & $-$ & $48.2 \pm 3.8$ & $198.3 \pm 42.8$ & $-$ \\
        VarSort & $-$ & $-$ & $3.4 \pm 2.0$ & $-$ & $-$ & $20.8 \pm 4.5$ \\
    \hline
    \end{tabular}
    
    \label{tab:5}
\end{table*}

\begin{table*}[h!]
    \centering
    \caption{Synthetic experiment for $d=50$ with Laplace noise}
    \begin{tabular}{c|cccccc} 
    \hline
    & ER1 & & & ER4 & & \\
    \hline
    & SHD & SID & $D_{top}(\pi, A)$ & SHD & SID & $D_{top}(\pi, A)$ \\
    \hline
        SCORE (ours) & $11.0 \pm 4.5$ & $71.8 \pm 50.2$ & $4.0 \pm 2.5$ & $\bf{128.1 \pm 7.9}$ & $1384 \pm 131$ & $19.8 \pm 3.5$   \\
        CAM & $\bf{10.1 \pm 3.4}$ & $\bf{66.1 \pm 47.9}$ & $-$ & $134.6 \pm 7.2$ & $\bf{1361 \pm 136}$ & $-$ \\
        GraN-DAG & $21.9 \pm 3.9$ & $165.7 \pm 46.2$ & $-$ & $138.3 \pm 8.8$ & $1603 \pm 166$ & $-$ \\
        VarSort & $-$ & $-$ & $8.1 \pm 4.2$ & $-$ & $-$ & $47.3 \pm 8.7$ \\
    \hline
    \end{tabular}
    
    \label{tab:6}
\end{table*}


\begin{table}[h!]
    \centering
    \caption{Run time (in seconds) comparison of the algorithms on ER1. The first row corresponds to the time spent for finding the topological order in our method. $(*)$ In order to run CAM on $50$ nodes within a reasonable time, we had to use preliminary neighbour search while restricting the maximum number of neighbours to  $20$ ~\cite{buhlmann2014cam}.}
    \begin{tabular}{c|ccc} 
    \hline
    & $d=10$ & $d=20$ & $d=50$  \\
    \hline
        SCORE order & $3.3 \pm 0.1$ & $8.5 \pm 0.8$ & $31 \pm 2.9$  \\
        SCORE & $6.3 \pm 0.2$ & $32.7 \pm 6.7$ & $257 \pm 17$  \\
        CAM & $30.1 \pm 3.7$ & $313 \pm 80$ & $1143 \pm 79^{(*)}$  \\
        GraN-DAG & $185 \pm 26$ & $357 \pm 47$ & $1410 \pm 73$ \\
    \hline
    \end{tabular}
    \label{tab:time}
\end{table}

\subsection{Real data}

We now compare the algorithms on a popular real-world dataset for causal discovery \cite{sachs2005causal} ($11$ nodes, $17$ edges and $853$ observations), as well as the pseudo-real dataset sampled from SynTReN generator \cite{van2006syntren} (Table~\ref{tab:7}). We can see that on Sachs, our method matches the SHD of CAM while improving the SID. On the SynTReN datasets, GraN-DAG seems to perform best, but the confidence intervals highly overlap.

\begin{table}[h!]
    \centering
    \caption{Comparison of several algorithms on the real world dataset Sachs and $10$ datasets sampled from SynTReN.}
    \begin{tabular}{c|cc|cc} 
    \hline
    & Sachs & & SynTReN & \\
    & SHD & SID & SHD & SID \\
    \hline
        SCORE & $12$ & $45$ & $36.2 \pm 4.7$ & $193.4 \pm 60.2$ \\
        CAM & $12$ & $55$ & $40.5 \pm 6.8$ & $152.3 \pm 48.0$ \\
        GraN-DAG & $13$ & $47$ & $34.0 \pm 8.5$ & $161.7 \pm 53.4$ \\
    \hline
    \end{tabular}
    
    \label{tab:7}
\end{table}
\vspace{-2mm}

\section{Related work}

\paragraph{Causal discovery for non-linear additive models} Many algorithms have been proposed in the past few years for the specific problem studied in this work. GraN-DAG~\cite{lachapelle2019gradient} aims to maximise the likelihood of the observed data under this model, and uses a continous contraint for the acyclicity of the causal graph, proposed in~\cite{zheng2018dags}, in order to use a continuous optimization method to find a first order stationary point of the problem. CAM~\cite{buhlmann2014cam} further assumes that the link functions $f_i$ in~\eqref{eq:3.1} also have an additive structure. They first estimate a topological order by greedily maximizing the data likelihood, and then prune the DAG using sparse regression techniques.


In the scope of linear additive models, \cite{ghoshal2018learning} first proposed an approach to provably recover, under some hypothesis on the noise variances, the causal graph in polynomial time and sample complexity. Their approach can be seen as an order-based method, where the ordering is estimated by sequentially identifying leaves based on an estimation of the precision matrix. In spirit, their method is closely related to ours. For instance, if the link functions $f_i$ in~\eqref{eq:3.1} are all linear, then the score of the joint distribution of $X$ is given by $s(x) = -\Theta x$, where $\Theta$ is the precision matrix. Hence, the score's Jacobian, which is used in our algorithm to identify the causal graph, can be seen as a non-linear generalization of the precision matrix, which has 
shown success for identifying causal relations in linear settings~\cite{loh2014high}.

While our work focuses on the identifiable non-linear additive Gaussian noise model, other works target more general non-parametric model, but must then rely on different kinds of assumptions such as faithfulness, restricted faithfulness or sparsest Markov representation~\cite{spirtes2000causation, raskutti2018learning, solus2021consistency}. These works apply conditional independence tests, and learn a graph that matches the identified  conditional independence relations~\cite{spirtes2000causation, zhang2008completeness}.

\paragraph{Score estimation} In the scope of generative modelling~\cite{song2019generative}, the score function is learned by fitting a neural network minimizing the empirical Fisher divergence~\cite{hyvarinen2005estimation}. While performing well in practice, such method is quite computationally expensive and requires tuning of several training parameters.


For our purpose, we chose to instead minimize the kernelized Stein discrepancy, since this approach provides a close form solution, allowing fast estimation at all observations. In practice, such method performs similarly as score matching while being much faster to compute. Asymptotic consistency of the Stein gradient estimator, and its relation to score matching were analyzed in~\cite{barp2019minimum}.



\section{Conclusion}

In this work, we demonstrated a new connection between score matching and causal discovery methods. We found that, in the case of non-linear additive Gaussian noise model, the causal graph can easily be recovered from the score function. In addition to generative models, this provides a new promising application for score estimation techniques. The proposed technique includes two modules: one that evaluates the (Jacobian of the) score, and one that prunes the final DAG given a topological order. Note that any score matching or pruning method can be plugged in to obtain a new practical algorithm.

\paragraph{Future work} One focus of this work was to build a fast algorithm for estimating a topological order, while avoiding the combinatorial complexity of searching over permutations, and the use of any heuristic optimization approaches. For this reason, we avoided using popular score matching algorithms developed for score-based generative modelling in high-dimensions~\cite{song2020improved}, since re-training a neural network after each leaf removal would be quite expensive in practice. Amortization~\cite{lowe2020amortized} is a promising direction to alleviate this issue.

In addition, we would like to further study the application of score matching causal discovery methods to generative model other than additive (Gaussian) noise. Due to the one-to-one correspondence between the score of a distribution and its density function, it should be possible to recover the graph from the score for any identifiable model. The question is hence: How to read the graph from the score function for a given model, and is there a universal way to do it that encapsulates a large class of models?

\section*{Acknowledgments}
This work was supported by the Swiss National Science Foundation (SNSF) under  grant number $407540\_167319$.

\bibliographystyle{unsrt}  
\bibliography{references}  

\appendix
\onecolumn

\section{Additional experiments}
\label{app:1}

We show here additional synthetic experiments. Tables~\ref{tab:8},~\ref{tab:9} and~\ref{tab:10} show the result for additive noise model with Gumbel noise on Erd\"os-Renyi graphs. Tables~\ref{tab:11},~\ref{tab:12} and~\ref{tab:12} show the results for Gaussian noise with Scale-free graphs.

\begin{table*}[h!]
    \centering
    \caption{Synthetic experiment for $d=10$ with Gumbel noise}
    \begin{tabular}{c|cccccc} 
    \hline
    & ER1 & & & ER4 & & \\
    \hline
    & SHD & SID & $D_{top}(\pi, A)$ & SHD & SID & $D_{top}(\pi, A)$ \\
    \hline
        SCORE (ours) & $\bf{1.1 \pm 1.2}$ & $\bf{4.5 \pm 5.0}$ & $\bf{0.4 \pm 0.5}$ & $\bf{21.7 \pm 2.9}$ & $\bf{35.3 \pm 7.4}$ & $\bf{0.3 \pm 0.4}$   \\
        CAM & $2.0 \pm 1.5$ & $6.1 \pm 5.8$ & $1.6 \pm 0.8$ & $27.2 \pm 1.8$ & $48.9 \pm 9.0$ & $3.8 \pm 2.5$ \\
        GraN-DAG & $2.1 \pm 1.9$ & $9.7 \pm 10.4$ & & $22.9 \pm 3.2$ & $43.2 \pm 11.7$ & \\
        VarSort &  &  & $1.9 \pm 0.8$ &  &  & $8.2 \pm 3.0$ \\
    \hline
    \end{tabular}
    \label{tab:8}
\end{table*}

\begin{table*}[h!]
    \centering
    \caption{Synthetic experiment for $d=20$ with Gumbel noise}
    \begin{tabular}{c|cccccc} 
    \hline
    & ER1 & & & ER4 & & \\
    \hline
    & SHD & SID & $D_{top}(\pi, A)$ & SHD & SID & $D_{top}(\pi, A)$ \\
    \hline
        SCORE (ours) & $\bf{3.3 \pm 2.6}$ & $\bf{12.0 \pm 11.5}$ & $\bf{0.7 \pm 0.9}$ & $\bf{52.9 \pm 4.4}$ & $\bf{205.5 \pm 35.5}$ & $\bf{5.1 \pm 1.6}$   \\
        CAM & $5.8 \pm 1.5$ & $24.6 \pm 13.0$ & $3.0 \pm 2.0$ & $57.1 \pm 4.2$ & $230.0 \pm 39.3$ & $10.7 \pm 5.8$ \\
        GraN-DAG & $7.4 \pm 2.5$ & $29.2 \pm 11.3$ & & $54.9 \pm 4.3$ & $239.5 \pm 43.6$ & \\
        VarSort &  &  & $3.8 \pm 1.7$ &  &  & $20.8 \pm 6.6$ \\
    \hline
    \end{tabular}
    \label{tab:9}
\end{table*}

\begin{table*}[h!]
    \centering
    \caption{Synthetic experiment for $d=50$ with Gumbel noise}
    \begin{tabular}{c|cccccc} 
    \hline
    & ER1 & & & ER4 & & \\
    \hline
    & SHD & SID & $D_{top}(\pi, A)$ & SHD & SID & $D_{top}(\pi, A)$ \\
    \hline
        SCORE (ours) & $\bf{11.3 \pm 4.6}$ & $\bf{68.2 \pm 45.1}$ & $4.1 \pm 2.5$ & $\bf{132.6 \pm 8.0}$ & $1390 \pm 132$ & $\bf{19.7 \pm 3.4}$   \\
        CAM & $\bf{11.0 \pm 3.7}$ & $69.7 \pm 48.8$ & $-$ & $141.1 \pm 6.7$ & $\bf{1350 \pm 137}$ & $-$ \\
        GraN-DAG & $22.5 \pm 4.2$ & $167.1 \pm 47.3$ & $-$ & $139.9 \pm 7.0$ & $1552 \pm 143$ & $-$ \\
        VarSort & $-$ & $-$ & $8.8 \pm 1.6$ & $-$ & $-$ & $45.5 \pm 8.0$ \\
    \hline
    \end{tabular}
    \label{tab:10}
\end{table*}

\begin{table*}[h!]
    \centering
    \caption{Synthetic experiment for $d=10$ with Gaussian noise on scale free graphs}
    \begin{tabular}{c|cccccc} 
    \hline
    & SF1 & & & SF4 & & \\
    \hline
    & SHD & SID & $D_{top}(\pi, A)$ & SHD & SID & $D_{top}(\pi, A)$ \\
    \hline
        SCORE (ours) & $\bf{0.3 \pm 0.6}$ & $\bf{2.7 \pm 5.8}$ & $\bf{0.1 \pm 0.3}$ & $\bf{4.6 \pm 1.7}$ & $\bf{21.5 \pm 9.6}$ & $\bf{0.5 \pm 0.9}$   \\
        CAM & $0.4 \pm 0.5$ & $2.8 \pm 3.6$ & $0.3 \pm 0.3$ & $9.6 \pm 2.0$ & $40.4 \pm 11.4$ & $4.1 \pm 1.6$ \\
        GraN-DAG & $1.4 \pm 1.0$ & $12.5 \pm 9.7$ & $-$ & $4.7 \pm 1.8$ & $23.0 \pm 7.3$ & $-$ \\
        VarSort & $-$ & $-$ & $2.8 \pm 1.7$ & $-$ & $-$ & $7.0 \pm 3.2$ \\
    \hline
    \end{tabular}
    \label{tab:11}
\end{table*}

\begin{table*}[h!]
    \centering
    \caption{Synthetic experiment for $d=20$ with Gaussian noise on scale free graphs}
    \begin{tabular}{c|cccccc} 
    \hline
    & SF1 & & & SF4 & & \\
    \hline
    & SHD & SID & $D_{top}(\pi, A)$ & SHD & SID & $D_{top}(\pi, A)$ \\
    \hline
        SCORE (ours) & $\bf{0.9 \pm 0.9}$ & $13.8 \pm 12.6$ & $0.7 \pm 0.6$ & $17.5 \pm 3.5$ & $179.2 \pm 23.8$ & $\bf{3.6 \pm 1.4}$   \\
        CAM & $\bf{0.9 \pm 0.9}$ & $\bf{12.9 \pm 14.0}$ & $\bf{0.5 \pm 0.4}$ & $26.4 \pm 3.9$ & $253.7 \pm 28.8$ & $4.6 \pm 3.2$ \\
        GraN-DAG & $3.2 \pm 1.9$ & $25.5 \pm 15.6$ & $-$ & $\bf{14.7 \pm 4.0}$ & $\bf{168.0 \pm 39.2}$ & $-$ \\
        VarSort & $-$ & $-$ & $7.4 \pm 2.5$ & $-$ & $-$ & $20.2 \pm 7.2$ \\
    \hline
    \end{tabular}
    \label{tab:12}
\end{table*}

\begin{table*}[h!]
    \centering
    \caption{Synthetic experiment for $d=50$ with Gaussian noise on scale free graphs}
    \begin{tabular}{c|cccccc} 
    \hline
    & SF1 & & & SF4 & & \\
    \hline
    & SHD & SID & $D_{top}(\pi, A)$ & SHD & SID & $D_{top}(\pi, A)$ \\
    \hline
        SCORE (ours) & $4.6 \pm 2.4$ & $132.6 \pm 75.8$ & $4.0 \pm 1.0$ & $68.3 \pm 3.6$ & $1724 \pm 109$ & $21.8 \pm 5.0$   \\
        CAM & $\bf{3.6 \pm 1.9}$ & $\bf{115.4 \pm 72.6}$ & $-$ & $85.3 \pm 4.2$ & $1935 \pm 99$ & $-$ \\
        GraN-DAG & $9.2 \pm 3.3$ & $281.8 \pm 129.8$ & $-$ & $\bf{63.8 \pm 9.7}$ & $\bf{1677 \pm 118}$ & $-$ \\
        VarSort & $-$ & $-$ & $21.0 \pm 4.0$ & $-$ & $-$ & $73.0 \pm 10.6$ \\
    \hline
    \end{tabular}
    \label{tab:13}
\end{table*}

\end{document}